\documentclass[letterpaper,10pt,conference]{ieeeconf}

\IEEEoverridecommandlockouts
\usepackage{amsmath,amssymb,amsfonts}
\usepackage{algorithmic}
\usepackage{graphicx}
\usepackage{textcomp}
\usepackage{xcolor}
\usepackage{booktabs}
\usepackage{cite}
\usepackage[draft]{hyperref}
\usepackage{bm}
\usepackage{microtype}
\usepackage{tikz}
\usepackage{needspace}

\definecolor{gbpanel}{RGB}{254,240,227}
\definecolor{gbink}{RGB}{9,20,49}

\hypersetup{bookmarks=false}
\makeatletter\let\labelindent\relax\makeatother
\usepackage{enumitem}

\newtheorem{proposition}{Proposition}
\newtheorem{definition}{Definition}

\newenvironment{proof}{\noindent\textit{Proof:}}{\hfill$\blacksquare$\vspace{1.5mm}}

\title{\LARGE \bf The Robot Is Not Its Description: GaugeBench for Representation Robustness in Morphology-Aware Policies \vspace{-1.5mm}}

\author{Rahath Malladi$^{1}$, Arshia Sangwan$^{2}$, Rajesh K. Gupta$^{1}$, Tauhidur Rahman$^{1}$
\thanks{$^{1}$University of California San Diego, $^{2}$New York University.}
\thanks{Corresponding author: Rahath Malladi, \texttt{ramalladi@ucsd.edu}.}}

\begin{document}

\IEEEaftertitletext{\vspace{-1.0\baselineskip}}
\maketitle
\enlargethispage{\baselineskip}
\thispagestyle{empty}
\pagestyle{empty}

\begin{abstract}
A robot description does more than specify a physical mechanism: it also encodes arbitrary conventions, such as joint-axis direction, joint-angle zero, and the order and names of links and joints. Morphology-aware policies consume interfaces built from these descriptions, yet cross-embodiment evaluation typically changes the robot while keeping those conventions fixed. This leaves a simple question unanswered: does behavior survive when the robot stays fixed but its description changes? GaugeBench isolates this case by rewriting a fixed mechanism under physically equivalent conventions, verifying that its physics and policy interface are preserved, and then evaluating the same policy weights. The result is stark: three MetaMorph policies score 4030.6 on 80 familiar robots, but only 51.6 when those same robots are equivalently re-described, while 98 genuinely held-out robots score 1489.6. A new description can therefore be more damaging than a new robot. Tracing the failure reveals that axis reversal alone reproduces the collapse, joint-angle zero changes are nearly harmless, and reordering lies between them; moreover, changing joint-state and torque coordinates alone is sufficient to cause the failure, while changing description-derived features alone is not. The same phenomenon appears in ModuMorph and an unrelated PyBullet framework. Yet it is not irreversible: exact two-description transport restores the original controller, and training across equivalent axis conventions raises retained return under axis reversal from 3.6\% to 80.6\%. Together, these results separate mechanism robustness from representation robustness and show that cross-embodiment evaluation should test both.

\end{abstract}

\section{Introduction}

One physical mechanism admits many numerical descriptions. Two files can agree on every link, mass, inertia, and joint placement and still disagree on the direction an engineer called positive for a hinge, on the pose that hinge calls zero, and on the order and names of parts. Classical rigid-body equations transform covariantly under a consistent change of generalized coordinates, so a textbook routine returns the same world-frame motion from either file~\cite{featherstone2008rigid, murray2017mathematical}. A robot is a mechanism. A description is one way of writing it down.

Cross-embodiment robot learning increasingly conditions a single policy on a structured robot description, and tests it by changing the robot. Robustness here has two distinct axes, and that practice measures only one of them. \emph{Mechanism robustness} asks whether a policy transfers from the machines it trained on to a new one, $M_{\mathrm{train}} \rightarrow M_{\mathrm{new}}$. \emph{Representation robustness} asks whether it preserves behavior when one fixed mechanism $M$ is expressed by two physically equivalent descriptions $D(M)$ and $D'(M)$, with the machine itself untouched. Existing cross-robot benchmarks probe the first while keeping one description convention per robot~\cite{parakh2025anybody}. GaugeBench isolates the second (Fig.~\ref{fig:overview}).

Nothing here assumes that an unconstrained neural policy should be invariant to a coordinate convention it never saw. In general, it has no reason to be. The question is what that costs an evaluation. Cross-embodiment results are read as evidence that a policy has learned something about robots, and we ask how much of that evidence rests on a convention the benchmark never varies. Answering it requires a mechanism that stays fixed while its description changes.

The answer is not graceful degradation. Across three MetaMorph policies, mean return falls from 4030.6 on 80 familiar robots to 51.6 on those same robots under certified equivalent descriptions, while 98 physically held-out robots from the same procedural generator return 1489.6. Across the three baseline policy families we evaluate, every policy loses more from re-describing a robot it was trained on than from meeting the held-out population. A new robot can be easier than a new description of the same robot. We call this comparison the \textbf{Same-Robot Paradox}.

Changing the conventions one at a time shows that a single one is enough. Reversing joint-axis directions retains 3.6 percent of the original return against 97.7 for joint-angle zeros and 35.8 for reordering, and driving axis reversal's two channels separately, joint-state and torque coordinates alone retain 2.5 percent while description-derived features alone retain 37.6 percent. The effect is not specific to one implementation, reproducing on ModuMorph and, under axis reversal alone, on released policies from an unrelated line of work on modular robot control running in PyBullet. Exact interface transport then recovers the original controller with unchanged weights, and training over randomly drawn equivalent axis conventions raises retained return under axis reversal from 3.6 to 80.6 percent.

\begin{figure*}[t]
\centering
\includegraphics[width=0.90\textwidth]{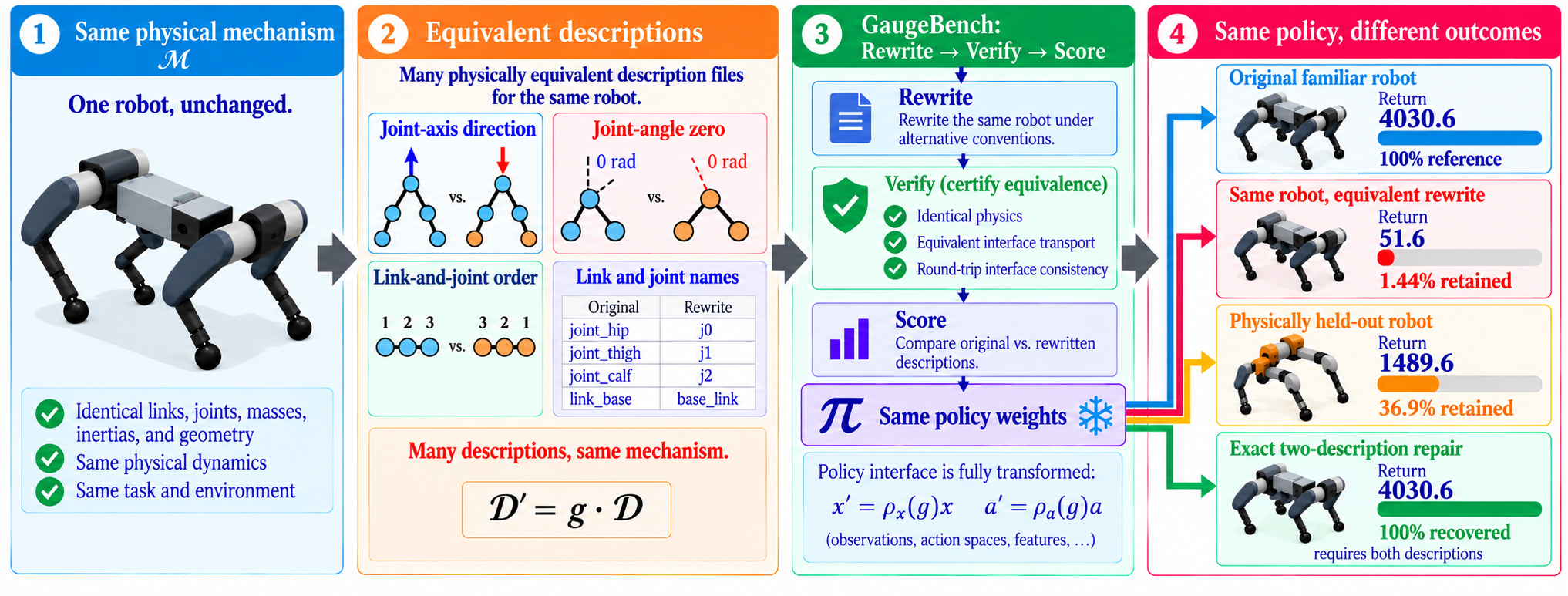}
\caption{\textbf{One mechanism, many descriptions, and the evaluation axis that exposes them.} A description records a mechanism together with arbitrary conventions: joint-axis directions $\sigma_j$, joint-angle zeros $q_{0,j}$, link-and-joint order $p$, and part names. GaugeBench rewrites a robot under a different choice of them, certifies physical and interface equivalence, and only then scores the \emph{same} weights. MetaMorph keeps roughly a third of its return on a physically held-out robot but almost nothing on an equivalent rewrite of a robot it trained on.}
\label{fig:overview}
\end{figure*}

\medskip\noindent\textbf{Contributions:} First, \textbf{an evaluation axis.} GaugeBench separates mechanism robustness from representation robustness, scoring fixed morphology-aware policies on descriptions whose physics and policy interface are certified equivalent before any episode is run. Second, \textbf{a same-robot representation failure, and which channel reproduces it.} An equivalent description of a familiar robot costs these policies more than a physically held-out robot, and controlled interventions show that changing joint-state and action coordinates alone reproduces the collapse while changing description-derived features alone does not. Third, \textbf{generality and response.} The failure reproduces on ModuMorph and on an independent PyBullet system. Exact transport recovers the original controller with unchanged weights, and axis-convention randomization raises retained return under axis reversal from 3.6 to 80.6 percent for the one convention family it samples.

\section{Related Work}

\subsection{Morphology-aware policy learning}
MetaMorph~\cite{gupta2022metamorph} tokenizes a kinematic graph for a Transformer, ModuMorph~\cite{xiong2023modumorph} modulates a shared backbone with per-robot context, and Body Transformer~\cite{sferrazza2024bodytransformer} masks attention by embodiment graph, each treating the description as a stand-in for the robot under one convention per machine. Current work extends the interface to morphology tokens~\cite{suzuki2026embedding} and morphology-conditioned world models~\cite{danesh2026morphology}, while cross-embodiment benchmarks evaluate transfer to novel embodiments~\cite{parakh2025anybody}. Policies conditioned on images and language~\cite{open_x_embodiment2023, pi0_2024} fix proprioception and action ordering in a comparable way, though we do not transform those interfaces and make no claim about them. We evaluate released systems whose interfaces can be reproduced and transformed exactly, which is what makes a causal test possible.

\subsection{Symmetry, equivariance, and canonicalization}
In classical rigid-body mechanics, these conventions are bookkeeping, and physical predictions do not depend on which one is chosen~\cite{featherstone2008rigid, murray2017mathematical}. Whether learned policies that consume descriptions inherit that independence is an empirical question. Four changes are easy to conflate. Cross-robot benchmarks vary the mechanism. Geometric deep learning handles a rigid motion of a whole scene under global $SE(3)$ symmetry~\cite{bronstein2021geometric}, with gauge-equivariant networks addressing tangent-frame choice on manifolds~\cite{cohen2019gauge}. Morphological symmetry exploits a symmetry a mechanism genuinely has, such as left--right limb pairs~\cite{wei2025beyond}. Ours is the fourth: a rewrite of one fixed mechanism's numbers, which leaves the mechanism and its asymmetries where they were and so exists for every robot, symmetric or not.

Repairing such a rewrite is standard machinery, and we claim no novelty for it. Conjugating a fixed policy by an exact interface transport is the lifted policy of an MDP isomorphism~\cite{vanderpol2020mdp}. Canonicalization maps an input to a canonical representative, applies a base model, and transports the output back, whether by frame averaging~\cite{puny2022frame}, a learned canonicalizer~\cite{kaba2023equivariance}, or an $SE(2)$ wrapper around a trained manipulation policy~\cite{deng2025eqbot}. Proposition~\ref{thm:repair} only specializes that construction to the full policy-visible input. What is new here is the evaluation axis: certifying physically equivalent descriptions before scoring, measuring what re-description costs against what a new mechanism costs, decomposing that cost by convention, and replicating it across independent systems.

\subsection{Closest recent work}
Tai~\cite{tai2026sameweights} shows that one vision--language--action checkpoint becomes a different controller when its action-scaling metadata changes, so fixed weights can depend on non-physical metadata. That change also alters the physical action produced, so the controller itself is not preserved. We instead certify physical equivalence before evaluation and put the loss in context against held-out robots. Yang and Hatton~\cite{yang2026coordinate} make inverse-dynamics identification independent of the generalized-coordinate chart. Their object is a model-fitting objective scored by prediction accuracy, whereas ours is a closed-loop policy scored by task behavior.

\section{Robot Description Equivalence}

\subsection{Descriptions and physical equivalence}
Let a mechanism $M$ be described by $D = (\mathcal{V}, \mathcal{E}, \{T_i\},$ $\{\hat{\bm{s}}_j\}, \{m_i\}, \{\bm{c}_i\}, \{I_i\}, \nu)$, with links $\mathcal{V}$, joints $\mathcal{E}$, parent-to-child frames $T_i \in SE(3)$, unit screw axes $\hat{\bm{s}}_j$, masses, centre-of-mass offsets, inertias, and an identifier assignment $\nu$ naming links and joints.

The evaluated policies do not consume $D$ directly. They consume a \emph{policy interface} built from it: per-part feature vectors in the order the file lists them, graph edges, padding masks, and per-joint features, alongside the robot state and task observation. Write $x$ for everything the policy sees and $\bm{a}$ for the action. The distinction carries the paper's causal result, because a description change can reach the interface through description-derived features, through joint-state and action coordinates, or both.

\begin{definition}[Physical equivalence]
A transformation $g$ carrying $D$ to $D' = g\cdot D$ is physically equivalent if it acts by an invertible change of generalized coordinates $A_g$, extended to joint velocities and torques, under which both descriptions realize the same mechanism in exact arithmetic. Chart-independent quantities agree directly: world-frame link and contact geometry, tip velocity, kinetic energy, and the trajectory a matched torque sequence produces. Chart-dependent ones agree only after transport by $A_g$, which is how spatial Jacobians, the mass matrix, inverse dynamics, and forward accelerations are compared, since each transforms covariantly.
\label{def:equiv}
\end{definition}

Physical equivalence is a property of the descriptions. What a policy meets is the pair of invertible conversions $\rho_x(g)$ and $\rho_a(g)$ that $A_g$ induces on the complete policy interface and on actions.

Definition~\ref{def:equiv} is a statement about exact arithmetic. What can be checked is an implementation of it. GaugeBench certifies each rewrite numerically over a 100-step rollout from the environment's reset pose under actuator commands drawn uniformly from $\pm 0.1$ per step. Each quantity in Definition~\ref{def:equiv} is checked separately rather than pooled, in its own units and after transport where it depends on the chart, and its largest componentwise absolute difference must stay within an absolute $10^{-7}$. Certification runs in the same simulator configuration used to score the policy, which is part of the claim, because a rewrite can preserve physics under one contact setting and not another, as Sec.~\ref{sec:independent} shows.

Exact mathematics and floating-point code differ slightly. Adding an offset to a joint angle and subtracting it again returns the original angle on paper, and in float64 returns it bitwise in 86 percent of cases from this family and within one unit in the last place on the rest. Conversion back to the original interface is therefore required to agree within a tolerance fixed from the available precision rather than bitwise.

Four description choices are varied. \textbf{Joint-axis direction} $\sigma_j \in \{\pm 1\}$ determines which direction around a joint is positive. \textbf{Joint-angle zero} $q_{0,j}$ is the offset that fixes which physical pose the rewritten file calls $q'_j = 0$. \textbf{Link-and-joint order} is the order in which parts are written in the file and passed to the policy, a permutation $p$. \textbf{Link and joint names} may be replaced bijectively, with every reference updated. Combined changes are verified together, because small numerical errors need not remain small after composition.

\subsection{Converting policy inputs and actions}
For changes to axis direction, angle zero, and link-and-joint order, joint coordinates and torques convert as
\begin{equation}
q'_j = \sigma_j\big(q_{p^{-1}(j)} - q_{0,j}\big), \quad
\dot{q}'_j = \sigma_j \dot{q}_{p^{-1}(j)}, \quad
\tau'_j = \sigma_j \tau_{p^{-1}(j)},
\label{eq:transport}
\end{equation}
so a change of zero shifts the angle but not its rate. Joint limits travel with their joints, mirrored by a sign flip and shifted by a change of zero, and the features computed from the description follow suit: part vectors are reordered, direction-dependent quantities change sign, and graph edges and padding masks take the new order. Rewriting a description without also converting observations and actions would change the physical controller rather than only its representation, so we verify Eq.~\eqref{eq:transport} for every case before running the policy.

\subsection{Exact repair}
Write $\rho_x(g)$ for conversion of the full policy input, $\rho_a(g)$ for conversion of actions, and $x' = \rho_x(g)\,x$.

\begin{proposition}[Full-input description repair]
Let $\pi$ be a fixed policy evaluated deterministically, and let $g$ admit invertible input and action conversions in exact arithmetic. The rewritten interface presents $x'_t = \rho_x(g)\,x_t$ from reset onward, converts each action back with $\rho_a(g)^{-1}$ before execution in the original coordinate system, and evaluates reward and termination on the original physical trajectory. For a correspondence $\hat{g}$ recovered from the two descriptions, define
\begin{equation}
\pi_{\mathrm{repair}}(x') \;=\; \rho_a(\hat{g})\;\pi\big(\rho_x(\hat{g})^{-1} x'\big).
\label{eq:repair}
\end{equation}
If $\hat{g} = g$, the closed loop formed by $\pi_{\mathrm{repair}}$ on $D'$ realizes the same physical trajectory as $\pi$ on $D$.
\label{thm:repair}
\end{proposition}

\noindent\textit{Proof:} For $\hat{g} = g$, the inverse input transport in~\eqref{eq:repair} recovers $x_t$, the policy therefore produces $\pi(x_t)$, and the output transport followed by the interface's inverse conversion recovers the same physical action. Induction on $t$ gives the original physical trajectory.

Two consequences matter. Repair depends on recovering the correct link-and-joint correspondence from the two files without access to their generator, and the conversion must cover \emph{every} policy input: correcting joint state while leaving description-derived features, edges, or masks in the rewritten order does not reproduce the original behavior. The proposition assumes exact arithmetic and bounds no floating-point error, and Sec.~\ref{sec:experiments} reports the measured residuals. For a policy with memory, it applies to the full input history while the hidden state stays in the original coordinates, which is how we repair the recurrent policy in Sec.~\ref{sec:independent}.

\section{GaugeBench}

GaugeBench applies metamorphic testing~\cite{chen1998metamorphic} to a robot policy in three steps: \textbf{rewrite} the robot under specified alternative conventions, \textbf{verify} that physics, observations, and actions remain equivalent under Definition~\ref{def:equiv}, and \textbf{score} the unchanged policy on the original description, on the rewritten description, and after repair. GaugeBench scores only rewrites that pass verification, and a failed repair stays in the average as a failure rather than being dropped. Verification residuals are reported with the relevant equivalence checks.

Repair is verified at the interface rather than by return. Observations and actions convert to the rewritten interface and back without running the simulator, on five probe states per case. The discrete channels, graph edges, padding masks, and traversals must reconstruct exactly; the continuous proprioceptive and context channels must agree within $10^{-14}$, and actions must round trip exactly. That $10^{-14}$ is a fixed multiple of float64 epsilon taken from the operation count of the transport, not from the residuals it judges. Full-episode return agreement is reported only as a diagnostic. Interface reconstruction determines whether conversion passed, because over a thousand steps a last-digit difference can change a contact event and produce a large return difference.

Three quantities carry the results, and they do not have equal standing. The primary endpoint is the \emph{same-robot loss} in native return units, the original score minus the rewritten score for the same robot and episode. The scale-free readout used in the figures is \emph{do-nothing-adjusted retained performance}. The contextual comparator is the \emph{new-robot loss}, the original score on the 80 evaluated robots minus the score on 98 held-out morphologies from the same procedural generator. Both losses start from the same 80-robot score and are therefore directly comparable.

Retained performance is computed for each robot, rewrite, and episode as the rewritten score above that robot's own do-nothing score, obtained with zero torque in place of the policy, divided by the original score above the same baseline, then averaged over episodes, rewrites, robots, and policies. Standing still is worth between $-20.1$ and $+26.7$ depending on the robot, averaging $-0.5$, so one shared baseline would distort both tails. The fifth percentile and worst case therefore stay in return units, as tail diagnostics rather than endpoints. Held-out bars divide held-out by original means above baseline.

A method scores nothing unless it covers every requested rewrite, and reports ambiguity rather than guessing. Every experiment carries a control supplied with the known correct mapping, states the simulator configuration it ran in, and repairs \emph{two-sided}, receiving both descriptions. Repair is \emph{one-sided} when given only the rewritten one, which remains future work.

% \vspace{-12pt}

\section{Experiments}
\label{sec:experiments}

\noindent\textbf{Statistical design.} The independently trained policy is the top-level replicate, and there are three per system. Robots, rewrites, and episodes replicate \emph{within} a policy: they characterize conditional variation and do not add independent trained policies. Top-level statements therefore carry a Student interval across policy-level values. Where an interval resamples the evaluation hierarchy instead, we say so, and it is conditional on these three policies. The three PyBullet policies are released artifacts differing in capacity rather than training seeds, so we do not read them as seed-level replication.

\subsection{The Same-Robot Paradox on MetaMorph}

We train official MetaMorph on the official UNIMAL-100 flat-terrain locomotion task for the published budget of 99.9M transitions, and require competence before testing any rewritten description: over the 100 training robots and 32 deterministic episodes each, mean per-robot return must reach 1000, the median 500, and at least 80 robots be positive. All three met it. The 98 test robots are disjoint from training but come from the same procedural generator, so they measure ordinary generalization rather than an arbitrary distribution shift. Rewrites are scored on 80 training robots, excluding a fixed 20-robot development set; the evaluated 80 are not easier, averaging 4030.6 versus 4226.6 on the excluded 20. Each rewrite reverses joint axes, changes joint-angle zeros, and reorders links and joints while keeping their names.

All three policies collapse together (Table~\ref{tab:multiseed} and Fig.~\ref{fig:paradox}a). Original return of roughly 4000 falls to roughly 50 under equivalent rewrites. Both tail readouts fall below zero, so some rewritten descriptions leave a gait that is actively counterproductive rather than broken. The primary endpoint is the paired same-robot loss, with a policy-level interval of $[3724.6,\ 4233.5]$, and as a descriptive secondary measure the policies retain 1.44 percent. For 9, 16, and 11 of the 80 robots, the rewritten-description policy performs worse than applying zero torque.

The held-out population sets a scale for that loss. Those 98 robots average 1489.6, so measured from the same original score the new-robot loss is 2541.1, and every policy loses more from rewriting a familiar robot than from meeting a new one. Both losses rest on that shared score, so their difference is descriptive. We attach no interval to the difference, because rewritten return already sits near its lower limit, so the difference mainly reflects how well each policy generalizes.

\begin{table}[t]
\caption{\textbf{MetaMorph under equivalent rewrites.} 3 policies named by
training seed, 80 robots, 10 rewrites, 32 episodes per case.}
\label{tab:multiseed}
\centering
\footnotesize
\setlength{\tabcolsep}{2.4pt}
\begin{tabular}{@{}lrrrr@{}}
\toprule
& \textbf{1409} & \textbf{1410} & \textbf{1411} & \textbf{Mean} \\
\midrule
\multicolumn{5}{@{}l}{\textit{Reference, original descriptions}}\\
\quad 100 training robots & 4048.8 & 4168.0 & 3992.6 & 4069.8 \\
\quad 98 held-out robots & 1291.6 & 1667.7 & 1509.4 & 1489.6 \\
\quad 80 evaluated robots & 4007.9 & 4145.2 & 3938.8 & 4030.6 \\
\midrule
\multicolumn{5}{@{}l}{\textit{Equivalent rewrite, same 80 robots}}\\
\quad Mean return & 51.4 & 54.3 & 49.1 & 51.6 \\
\quad 5th percentile & $-39.1$ & $-31.6$ & $-25.1$ & $-31.9$ \\
\quad Worst case & $-53.7$ & $-42.6$ & $-36.0$ & $-44.1$ \\
\quad With corrected normalization & 54.4 & 65.5 & 60.1 & 60.0 \\
\midrule
\multicolumn{5}{@{}l}{\textit{Comparison, common baseline}}\\
\quad Same-robot loss & 3956.5 & 4090.9 & 3889.8 & 3979.1 \\
\quad Retained (\%) & 1.46 & 1.44 & 1.42 & 1.44 \\
\quad New-robot loss & 2716.3 & 2477.5 & 2429.4 & 2541.1 \\
\quad Excess same-robot loss & $+$1240.2 & $+$1613.4 & $+$1460.4 & $+$1438.0 \\
\bottomrule
\end{tabular}
\end{table}

\begin{figure*}[t]
\centering
\includegraphics[width=0.88\textwidth]{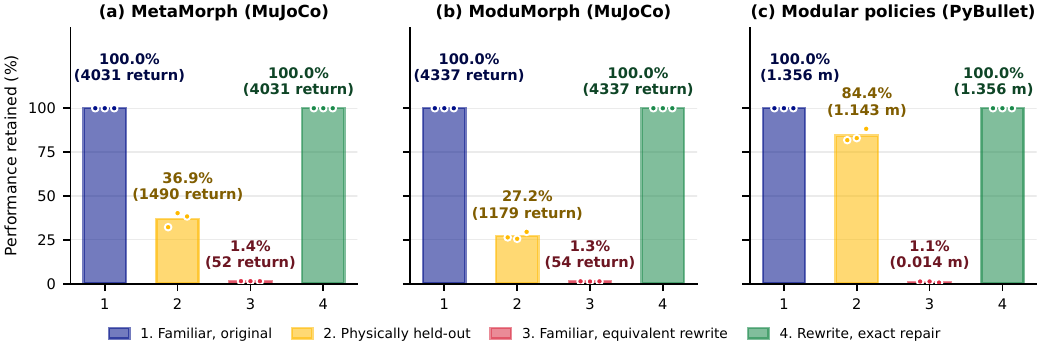}
\caption{\textbf{An equivalent re-description of a familiar robot is more damaging than a physically held-out robot, in all three baseline policy families.} Each panel holds one family's trained policies fixed across the four conditions numbered beneath it: familiar robots under their original descriptions, physically held-out robots, familiar robots under equivalent rewrites, and those rewrites after repair. Only the second changes the mechanism. Bars are performance retained above a do-nothing score, zero for panel (c), averaged over policies, with the absolute score beneath, and points are the individual policies. Transport restores 100.0\% aggregate retained performance, matching the original episode for episode except in five of 25{,}600 paired episodes in one ModuMorph policy. Panel (c) rewrites axis directions alone, because reordering is not certified physically equivalent in that simulator, and its three policies are released artifacts differing in capacity rather than training seeds.}
\label{fig:paradox}
\end{figure*}

Two controls narrow the explanation. Correcting the stored mean and variance used to scale observations raises return only from 51.6 to 60.0, closing 0.2 percent of the 98.7 percent gap, so stale normalizer statistics are not the cause. Supplying the correct link-and-joint correspondence instead makes the unchanged weights reproduce the original return exactly in all 2{,}400 cases, as does a correspondence recovered from the description bytes alone. The weights keep their competence, and what changed is how the interface presents inputs and reads actions.

\subsection{Which convention breaks the policy}
\label{sec:decomposition}

The main experiment changes three conventions at once, so it cannot say which one matters. We therefore vary each convention separately using the same three policies and 80 robots. For each of the three conventions, we evaluate five rewrites per robot and 32 episodes per case, giving 1,200 unique robot-rewrite pairs, each evaluated by the same 3 policies. All pairs remain physically equivalent, with a largest verification error of $1.86\times10^{-10}$, and none of the changes is cosmetic. Each actuated joint flips independently with probability one half, realizing 46.5 percent flipped over 5 to 15 joints; zero offsets are sampled uniformly from $\pm0.5$~rad and realize $0.259$~rad mean absolute; and the permutation leaves 14.7 percent of movable tokens in place. Percentages are computed within each policy and then averaged. Unless stated otherwise, intervals resample policies, robots, rewrites, and episodes hierarchically and are therefore conditional on these three trained policies.

Reversing joint-axis directions alone is sufficient to reproduce the collapse (Table~\ref{tab:decomposition}). It retains 3.6 percent against 1.9 percent when all three conventions change, reaching approximately the same floor as the combined rewrite. The combined column uses the five rewrite seeds drawn for this comparison, which is why it reads 1.9 percent rather than the 1.44 percent in Table~\ref{tab:multiseed}, which averages over all ten primary rewrites. For every policy, the axis-reversal loss exceeds the loss on new robots: the excess is $[1022,1586]$ under hierarchical resampling and $[825,1783]$ under a Student interval across policies. Changing joint-angle zeros costs only 91.5 return, or 2.3 percent, with intervals $[62,121]$ and $[44,139]$ that exclude zero, so the effect is small but measurable. Link-and-joint order is intermediate at 35.8 percent; its loss differs from the new-robot loss by only $+4.5$, with a Student interval of $[-428,437]$, so we neither claim nor rule out that reordering alone establishes the paradox, though its upper bound remains below the lower bound for axis reversal. In the final block of Table~\ref{tab:decomposition}, we report \emph{excess same-robot loss}, defined as same-robot loss minus new-robot loss, so positive values mean that rewriting a familiar robot is more damaging than encountering a genuinely new one.

\begin{table}[t]
\caption{\textbf{Effect of individual description conventions.}}
\label{tab:decomposition}
\centering
\scriptsize
\setlength{\tabcolsep}{2.0pt}
\begin{tabular}{@{}lrrrr@{}}
\toprule
& \textbf{\shortstack{Joint-axis\\direction}} & \textbf{\shortstack{Joint-angle\\zero}} & \textbf{\shortstack{Link-and-\\joint order}} & \textbf{Combined} \\
\midrule
Rewritten return    & 146.6 & 3939.2 & 1445.9 & 78.0 \\
Retained (\%)       & 3.6 & 97.7 & 35.8 & 1.9 \\
5th percentile      & $-17.8$ & 3773.3 & 612.7 & $-19.3$ \\
Worst case          & $-26.3$ & 3748.1 & 528.7 & $-25.2$ \\
\midrule
Same-robot loss     & 3884.1 & 91.5 & 2584.7 & 3952.6 \\
\quad 95\% interval & $[3726,4038]$ & $[62,121]$ & $[2434,2733]$ & $[3793,4110]$ \\
\midrule
\multicolumn{5}{@{}l}{\textit{Excess same-robot loss, same-robot minus
  new-robot}}\\
\quad Estimate      & $+1303.8$ & $-2488.8$ & $+4.5$ & $+1372.4$ \\
\quad 95\% interval & $[1022,1586]$ & $[-2736,-2253]$ & $[-274,280]$ & $[1087,1657]$ \\
\bottomrule
\end{tabular}
\end{table}

These outcomes fit the architectures, which use learned position embeddings and carry no invariance to reversing joint coordinates. The convention-wise effects are not additive: axis reversal and the combined change both reach the performance floor, so their losses cannot tell us what fraction of the combined failure is attributable to axis reversal, and interactions between conventions prevent interpreting the one-at-a-time effects as components of a sum. Fig.~\ref{fig:distribution} further shows that the collapse is not driven by a small number of unusually fragile robots.

\subsection{Which part of the interface reproduces the failure}

Which channel is sufficient to reproduce the collapse? Reversing a joint axis changes two parts of the policy interface at the same time: the joint limits and axis features derived from the description, and the coordinate system in which joint states are observed and torques are commanded. We drive each channel on its own. Neither partial change is internally consistent, and neither describes a physically valid robot, so these are diagnostic interventions, outside GaugeBench, and are not scored as equivalent-description tests.

The channels separate cleanly. Changing only joint-state and torque coordinates retains 2.5 percent, reproducing the collapse through that channel alone. Changing only the description-derived features retains 37.6 percent, and every policy keeps walking, with a fifth-percentile return of $+547.3$ against $-38.4$ when only state and torque coordinates change. Three losses then fall within 3 percent of each other: 2541.1 on held-out robots, 2584.7 after reordering, and 2517.6 after changing only the axis-related description features.

Changing the coordinate interface alone is therefore sufficient to reproduce the catastrophic collapse, and changing the description-derived features alone is not, costing instead about what ordinary generalization to a new robot costs. The extreme sensitivity implicates the coordinate interface through which the policy observes and acts on the mechanism rather than a failure to read morphology features. Because the coordinate-only and full axis-reversal conditions both lie near the floor, and the interventions are not additive, their magnitudes do not apportion the combined loss between the channels, and we do not interpret the ordering of 3930.6 against 3884.1 for the complete axis reversal.

\begin{figure*}[t]
\centering
\includegraphics[width=0.85\textwidth]{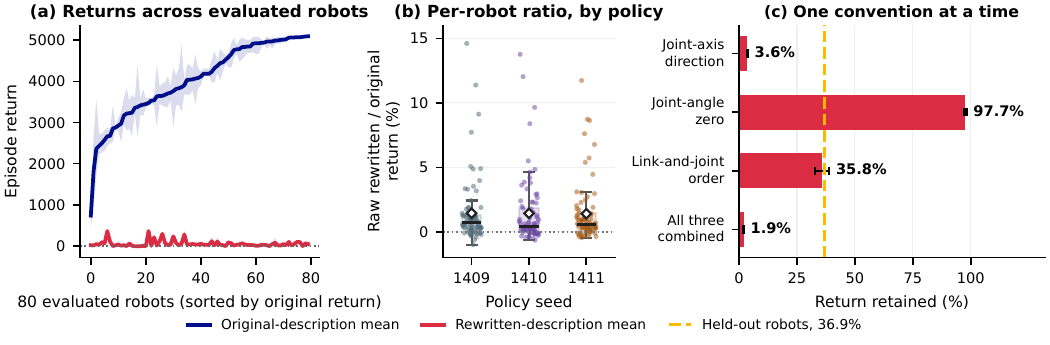}
\caption{\textbf{The collapse is uniform across robots and policies, and one convention alone reproduces it.} (a) Per-robot return on the 80 evaluated robots, sorted by original performance, with bands spanning the three policies. Rewritten return stays near zero across the entire range, so the mean is not carried by a few fragile robots. (b) Per-robot raw ratio of rewritten to original return, one column per trained policy. All 240 robot--policy pairs fall below 15 percent, and the axis is not truncated. (c) Do-nothing-adjusted return retained per convention, against the held-out reference. Joint-axis direction alone reaches the floor the combined rewrite reaches, while the joint-angle zero is nearly harmless. Error bars are the between-policy standard deviation, and these effects are not additive.}
\label{fig:distribution}
\end{figure*}

For axis reversal and reordering, the mapping recovered from the descriptions is bitwise identical to the supplied one. Across all 1{,}200 cases, every discrete value reconstructs exactly; the largest continuous error is $6.66\times10^{-16}$ against the $10^{-14}$ bound, and actions round trip exactly.

\subsection{The same failure in two other systems}

\noindent\textbf{ModuMorph.} To test whether this is specific to one morphology encoder, we trained official ModuMorph at the same policy seeds, budget, task, and observation, action and normalizer contract, on the same morphologies, and re-described the same 80. All three policies are competent under the criterion above, and the design is 3 policies, 80 morphologies, 10 transformations, and 32 episodes per case (Fig.~\ref{fig:paradox}b). Original return is 4336.5 against 54.1 after rewriting, a mean same-robot loss of 4282.5, with per-policy values of 4249.0, 4305.8, and 4292.6 and a policy-level interval of $[4208.6, 4356.3]$. The policies retain 1.32 percent. For 7, 9, and 11 of the 80 robots, they do worse than applying zero torque, and every policy again loses more to a rewritten familiar robot than to a held-out one. Both known and structure-recovered mappings restore all three. Agreement under the known mapping is a diagnostic: two policies reproduce all 25{,}600 paired episodes exactly and the third reproduces 25{,}595, the five exceptions differing by 7.4 return after a last-digit rounding difference reaches a contact event. The two official configurations differ in several ways at once, so this shows the failure is not confined to one robot encoder, without isolating which difference matters.

\noindent\textbf{Independent-system replication.}
\label{sec:independent}
MetaMorph and ModuMorph share an ecosystem, so we ran GaugeBench against released policies from an unrelated line of work on modular robot control~\cite{whitman2023modular}. The experimental stack is independent of the evidence above: PyBullet~\cite{coumans2021pybullet} rather than MuJoCo~\cite{todorov2012mujoco}, hardware modules rather than procedural bodies, model-based RL rather than PPO, a per-module recurrent graph network rather than a Transformer, and commanded body-velocity tracking rather than forward locomotion reward. Twelve trained designs and 10 transformations give 120 design--rewrite pairs, each evaluated by 3 released policies over 8 episodes, so 360 policy-evaluation cases. Twelve held-out designs are drawn deterministically from the remaining compiled designs.

Here we reverse joint-axis directions only, using the same independent coin flip per joint as in the MetaMorph experiment. Serialization order is excluded rather than reported uncertified, because this simulator orders contact pairs by link index under \texttt{URDF\_USE\_SELF\_COLLISION}, so reordering could change the physics rather than preserve it. Reversing an axis and mirroring its joint limits describes the same physical hinge using the opposite positive direction, and under the framework's own simulator settings all 120 pairs certify, the largest base-state error and the largest converted joint-state error both being exactly $0.0$.

The paradox reproduces (Fig.~\ref{fig:paradox}c). Progress along the commanded direction averages $1.356$~m with original descriptions and $1.143$~m on a new design, but $0.014$~m, or 1.05 percent, after rewriting a familiar one. The same-robot loss of $1.342$~m is 6.3 times the new-robot loss of $0.213$~m. Per policy, the difference is $+1.12$, $+1.13$ and $+1.14$~m, with resampling over designs, rewrites and episodes giving $[0.986, 1.272]$~m conditional on these released policies. New designs cost these policies far less than they cost MetaMorph, yet rewritten descriptions reduce both systems to about one percent, and 157 of the 360 policy-evaluation cases move opposite the commanded direction. Both mapping-based repairs are exact. Matching by part name cannot repair axis reversal, because names do not record axis direction, a limit of that matching method rather than of the repair problem.

This framework also isolates the coordinate channel by construction, since axis reversal changes joint states and actions while module types and the attachment graph stay identical across all 120 pairs. A description-blind controller from the same release, a shared network selecting layers by design index, behaves the same way, travelling $0.359$~m originally and $-0.007$~m after axis reversal, though its original score is a quarter of the modular policies' and the comparison is directional only.

\subsection{Training across the axis convention family}
\label{sec:defense}

A direct mitigation is to stop treating one axis convention as privileged. We therefore train three additional MetaMorph policies matched one-to-one with the three baselines by training seed, using the same robots, budget, architecture, optimizer, task, and morphology set. The only change is exposure to axis conventions: at the start of each training episode, with probability one half, the robot is rewritten under a randomly drawn equivalent axis convention from the same family used by GaugeBench. Evaluation draws are independent of those seen during training, and all three policies satisfy the same competence criterion as the baselines. Table~\ref{tab:defense} reports means across the three matched policy pairs; the Change column is the paired axis-randomized minus standard difference. Randomized augmentation itself is standard machinery. The question here is whether exposure to equivalent conventions is enough to remove the representation fragility revealed above.

The effect is large. Under axis reversal, the randomized policies retain 80.6 percent of their original-description return, compared with only 3.6 percent for the standard policies. The three matched policies retain 73.7, 86.9, and 81.3 percent, and a Student interval on the matched-seed improvement is $[60.6,93.3]$ percentage points. The lower tail improves with the mean: the worst robot returns 2519.4 rather than $-26.3$, leaving no evaluated robot near the do-nothing floor. For axis reversal, the Same-Robot Paradox therefore reverses: these policies lose less from rewriting a familiar robot than from transferring to the held-out population.

The mitigation is neither free nor complete. Original-description return falls from 4030.6 to 3606.3, a decrease of 424 on average, or about 10 percent, with a matched-seed interval of $[-616,-232]$. Held-out performance is essentially unchanged at the resolution of three policies, moving from 1489.6 to 1491.2. Robustness also remains incomplete outside the convention family seen during training: when axis reversal is combined with the two untreated convention changes, retained return rises from 1.9 to only 30.9 percent. This is consistent with substantial residual sensitivity to joint-angle-zero and ordering conventions.

The gain is not simply memorization of evaluation-time axis patterns. Training samples from the convention family rather than enumerating it: for the median evaluated robot, the probability that the exact axis pattern used at evaluation ever appeared during training is only 0.27. The recovered performance therefore reflects generalization across equivalent axis conventions rather than recall of specific descriptions.

\begin{table}[t]
\caption{\textbf{Axis-convention randomization.}}
\label{tab:defense}
\centering
\footnotesize
\setlength{\tabcolsep}{3.0pt}
\begin{tabular}{@{}lrrr@{}}
\toprule
& \textbf{Standard} & \textbf{\shortstack{Axis-\\randomized}} & \textbf{Change} \\
\midrule
Original-description return & 4030.6 & 3606.3 & $-424$ \\
Axis-reversed return & 146.6 & 2911.8 & $+2765$ \\
\textbf{Return retained} & \textbf{3.6\%} & \textbf{80.6\%} & $+77.0$ pts \\
Combined rewrite & 1.9\% & 30.9\% & $+28.9$ pts \\
Held-out return & 1489.6 & 1491.2 & $+2$ \\
\bottomrule
\end{tabular}
\end{table}

The two responses address different failure modes. Exact transport diagnoses representation misalignment when both descriptions are available and demonstrates that the frozen policy remains competent. Axis randomization instead asks whether robustness to an entire convention family can be acquired during training; once trained this way, the policy no longer requires the original description at evaluation time for the convention family it has learned to tolerate.

\subsection{Recovering the correspondence without shared names}

Repair so far assumed shared identifiers, where name matching suffices, recovering $800/800$ with maximum zero-offset error $2.2\times10^{-16}$~rad. Descriptions from sources that never agreed on a namespace are the harder case, so we repeat the three convention changes on the same 80 robots while replacing every link and body name with a unique arbitrary one. The design is 3 policies, 80 robots, 5 rewrites, and 32 episodes, and all 400 robot--rewrite pairs pass the equivalence check, with a largest physics error of $1.4\times10^{-10}$. Mean original return is 4020.2 and falls to 34.1, a same-robot loss of 3986.1, which we do not read as worse than the 3979.1 of the stable-identifier family, since this family draws different rewrite seeds and different evaluation episode seeds on the same three trained policies. Every policy orders the two penalties the same way, so the paradox survives the harder family.

Renaming removes the shortcut but not the repair (Fig.~\ref{fig:gauntlet}). Name matching recovers $0/400$ and correctly declines to guess. Structure matching recovers $400/400$ and restores the original return exactly, against 0.9 percent with no repair. It matches parts using the structure of the two descriptions, their rooted kinematic trees and the physical attributes our rewrites leave unchanged, which is canonical labeling of attributed rooted trees~\cite{aho1974design}, and reports ambiguity where identical branches admit several mappings. Every evaluated robot is structurally asymmetric, so no such case arose, and refusal is exercised only by a unit test on a deliberately symmetric description, so repair is not validated under ambiguous symmetry. Other matching methods could do the same, and we make no claim that it is necessary.

\begin{figure}[t]
\centering
\includegraphics[width=0.92\columnwidth]{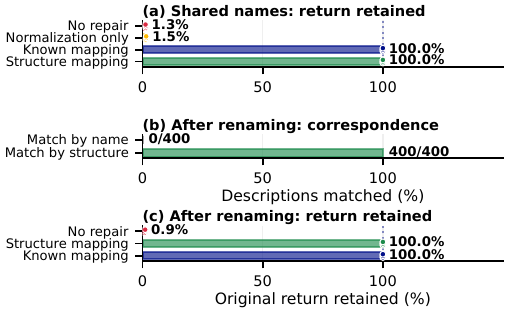}
\caption{\textbf{Renaming defeats the trivial fix, but not repair itself.} (a) With shared part names, a known correspondence and one recovered from description structure both restore the original return exactly, while correcting observation normalization alone does almost nothing. (b) With every name replaced, name matching recovers nothing and correctly declines to guess, while structure matching recovers all of them, and (c) repair there restores the original return exactly. Bars are means over the three policies, points the individual policies. Every repair is \emph{two-sided}, receiving both descriptions and retraining nothing.}
\label{fig:gauntlet}
\end{figure}

\section{Discussion and Limitations}

Mechanism robustness, $M_{\mathrm{train}} \rightarrow M_{\mathrm{new}}$, does not imply representation robustness, $D(M) \rightarrow D'(M)$. Every policy in the three primary comparison families transfers to a genuinely new machine better than it survives an equivalent rewrite of one it already knows. A policy that changes behavior under a physically equivalent description of a familiar robot is responding to more than the robot it controls.

A description gives a policy structural features describing the robot and the coordinate conventions in which joint states and actions are expressed. For MetaMorph, where we drove the two separately, changing the first alone costs about what ordinary transfer to a new robot costs, while changing the second alone reproduces the collapse. It is the second that makes an equivalent rewrite of a familiar robot worse than an unfamiliar one. Exact transport shows the weights keep their competence throughout, and training exposure shows part of this robustness can be learned, without solving the general case.

Several limits bound the claims. The evidence is simulation only, covers the evaluated morphology-aware policy families rather than every current method, and rests on three independently trained policies per MuJoCo system, which is the sample size behind every top-level interval. Simulation is also what makes the comparison tight: the description changes while the mechanism is held exactly fixed, where hardware would add calibration, wear, and sensing differences a description-only test cannot absorb. Repair is two-sided throughout, so recovery from the rewritten description alone remains open, as does policy-level repair where identical branches make the correspondence ambiguous. The training-time mitigation randomizes one convention family, leaves the other two untreated, and is not a claim of invariance. We do not estimate how often the evaluated convention changes occur in deployed robot-description pipelines, and GaugeBench measures their consequences when an equivalent description is used rather than present-day prevalence.

Morphology-aware cross-robot policies, whose interfaces expose joints and parts explicitly, should therefore be evaluated not only on new mechanisms but also on equivalent descriptions of mechanisms they already know, with physical equivalence certified before anything is scored. Whether policies that carry the robot interface implicitly behave the same way is open, and the axis is well defined for them too.

\bibliographystyle{IEEEtran}
\bibliography{references}

\end{document}